\documentclass{ifacconf}

\usepackage{natbib}
\usepackage{glossaries}
\usepackage{amsmath,amssymb,amsfonts}
\usepackage{graphicx}
\usepackage{algorithm,algorithmic}
\usepackage{textcomp}
\usepackage{enumerate}

\glsdisablehyper

\everymath=\expandafter{\the\everymath\displaystyle}
\usepackage{times}
\usepackage{xcolor}
\usepackage{scalerel,stackengine,mathtools,bm}
\usepackage{dsfont}
\usepackage{pgfplots}
\usepackage{siunitx}
\usepackage{xfrac}
\usepackage{caption,subcaption}
\usepackage{xparse}
\usepackage{xstring}
\usepackage{todonotes}
\usepackage{url}
\makeatletter\let\cl@part\relax\makeatother
\usepackage[capitalize,noabbrev]{cleveref}
\usepackage{booktabs}
\usepackage{tabularx}
\Crefname{part}{}{}

\newcommand{\hil}{\mathbb{H}}
\newcommand{\bhil}{\widetilde{\mathbb{H}}}
\newcommand{\bhilh}{\widehat{\mathbb{H}}}
\newcommand{\x}{\bm{x}}
\newcommand{\f}{\bm{f}}
\newcommand{\y}{\bm{y}}

\newcommand{\K}{\bm{K}}
\newcommand{\w}{\bm{w}}
\newcommand{\h}{\bm{h}}
\newcommand{\g}{\bm{g}}
\newcommand{\bmu}{\bm{\mu}}
\newcommand{\bsig}{\bm{\sigma}}
\newcommand{\bvphi}{\bm{\varphi}}

\newcommand{\tX}{\widetilde{X}}
\newcommand{\ty}{\widetilde{\bm{y}}}
\newcommand{\ts}{\textsuperscript}

\renewcommand{\footnotesize}{\scriptsize}

\newcolumntype{L}{>{$}l<{$}} 
\newcolumntype{C}{>{$}c<{$}} 

\makeatletter
\DeclareRobustCommand{\qed}{%
  \ifmmode 
  \else \leavevmode\unskip\penalty9999 \hbox{}\nobreak\hfill
  \fi
  \quad\hbox{\qedsymbol}}
\newcommand{\openbox}{\leavevmode
  \hbox to.77778em{%
  \hfil\vrule
  \vbox to.675em{\hrule width.6em\vfil\hrule}%
  \vrule\hfil}}
\newcommand{\qedsymbol}{\openbox}
\newenvironment{proof}[1][\proofname]{\par
  \normalfont
  \topsep6\p@\@plus6\p@ \trivlist
  \item[\hskip\labelsep\itshape
    #1.]\ignorespaces
}{%
  \qed\endtrivlist
}
\newcommand{\proofname}{Proof}
\makeatother

\usepackage{tikz}
\usetikzlibrary{matrix}
\usetikzlibrary{patterns,patterns.meta}
\usetikzlibrary{snakes}
\usetikzlibrary{decorations.pathmorphing,decorations.markings}
\usetikzlibrary{positioning,calc}
\usetikzlibrary{patterns,patterns.meta,snakes,shapes}
\usetikzlibrary{decorations.pathmorphing,decorations.markings}
\usetikzlibrary{calc}
\usetikzlibrary{arrows.meta, chains, scopes}
\usetikzlibrary{fit,backgrounds}
\usepackage{pgf}
\usepackage{pgfplots}
\usepackage{mathcomp}
\pgfplotsset{compat=1.17}
\usepgfplotslibrary{fillbetween}

\newacronym{desy}{DESY}{Deutsches Elektronen-Synchrotron}
\newacronym{euxfel}{European XFEL}{European X-Ray Free-Electron Laser}
\newacronym{xfel}{XFEL}{X-Ray Free-Electron Laser}
\newacronym{lbsync}{LbSync}{laser-based optical synchronization}
\newacronym[longplural={linear matrix inequalities}]{lmi}{LMI}{linear matrix inequality}
\newacronym{lo}{LO}{local oscillator}
\newacronym{lti}{LTI}{linear time-invariant}
\newacronym{mlo}{MLO}{master laser oscillator}
\newacronym{mo}{MO}{master timing reference oscillator}
\newacronym{pi}{PI}{proportional-integral}
\newacronym{pll}{PLL}{phase-locked loop}
\newacronym{ppl}{PPL}{pump-probe laser}
\newacronym{rf}{RF}{radio-frequency}
\newacronym{rms}{RMS}{root-mean-square}
\newacronym{slo}{SLO}{subsidiary laser oscillator}
\newacronym{snr}{SNR}{signal-to-noise ratio}
\newacronym{vco}{VCO}{voltage controller oscillator}
\newacronym{inr}{INR}{improvement-to-noise ratio}
\newacronym{bo}{BO}{Bayesian optimization}
\newacronym{mtbo}{MTBO}{multi-task Bayesian optimization}
\newacronym[longplural={Gaussian processes}]{gp}{GP}{Gaussian process}
\newacronym{ei}{EI}{expected improvement}
\newacronym{rkhs}{RKHS}{reproducing kernel Hilbert space}
\newacronym{ucb}{UCB}{upper confidence bound}
\newacronym{lcb}{LCB}{lower confidence bound}
\newacronym{pdf}{PDF}{probability density function}
\newacronym{cdf}{CDF}{cumulative distribution function}
\newacronym{smgo}{SMGO-\(\delta\)}{set membership global optimization}
\newacronym{awgn}{AWGN}{additive white Gaussian noise}
\newacronym{psd}{PSD}{power spectral density}
\newacronym{lsu}{LSU}{link stabilization unit}
\newacronym{oxc}{OXC}{optical cross-correlator}
\newacronym{icm}{ICM}{intrinsic co-regionalization model}
\newacronym{lmc}{LMC}{linear model of co-regionalization}
\newacronym{mle}{MLE}{maximum likelihood estimation}
\newacronym{mcmc}{MCMC}{Markov chain Monte Carlo}
\newacronym{map}{MAP}{maximum a posteriori} 
\newacronym{lkj}{LKJ}{Lewandowski-Kurowicka-Joe}

\newif\ifarxiv
\arxivtrue

\begin{document}
\begin{frontmatter}

\title{
Safe Bayesian Optimization for Uncertain Correlation Matrices in Linear Models of Co-Regionalization
}


\author[First]{Jannis Lübsen}
\address[First]{J. Lübsen is with the Institute of Control Systems, Hamburg University of Technology, Hamburg, Germany
        {\tt\small jannis.luebsen@tuhh.de}}%
\author[Second]{Annika Eichler}
\address[Second]{A. Eichler is with Institute of Control Systems, Hamburg University of Technology and with the Deutsches Elektronen-Synchrotron, Hamburg, Germany
        {\tt\small annika.eichler@tuhh.de}}%

\begin{abstract}
This paper extends safety guarantees for multi-task Bayesian optimization with uncertain co-regionalization matrices from intrinsic co-regionalization models to linear models of co-regionalization. The latter allows for more flexible modeling of the inter-task correlations by composing multiple features. We derive uniform error bounds for vector-valued functions sampled from a Gaussian process with a linear model of co-regionalization kernel. Furthermore, we show the potential performance gains of linear models of co-regionalization in a numerical comparison on a safe multi-task Bayesian optimization benchmark.
\end{abstract}

\begin{keyword}
Gaussian Processes, Bayesian Optimization, Safety Guarantees, Multi-Task Gaussian Processes, Simulation Informed
\end{keyword}

\end{frontmatter}


\section{Introduction}
PID controllers are still the most widely used controllers in industry. However, the tuning of these controllers is often done by trial and error, which can be time-consuming and costly. To overcome this issue, the field of \gls{bo} has gained popularity in recent years. \gls{bo} is a black-box optimization technique, i.e., initially almost no prior knowledge is necessary for setting up the optimization procedure. Then, by sequential evaluations of the unknown objective function, the algorithm constructs a probabilistic surrogate model which is used to predict the function values of the objective at unobserved points. These predictions are then used to identify new promising points which are likely to optimize the objective, see \cite{Shahriari2016} for an overview of \gls{bo} in general. By now, this optimization method has been established and is in practice widely used to automate parameter-tuning processes. For instance \cite{Coutinho2023,Hajieghrary2022} applied \gls{bo} to automatically tune and adapt PID controllers, and \cite{Shields2021} used \gls{bo} to optimize ingredients of chemical reactions. Furthermore, \gls{bo} became a cornerstone for optimizing large scale research infrastructures such as particle accelerators, see \cite{Kirschner2019,Kaiser2024}.

Despite the great success of \gls{bo}, the method often becomes infeasible when applied to high-dimensional problems. In general, a \gls{gp}, see \cite{Williams2006}, is utilized to model the unknown objective which suffers from the curse of dimensionality, i.e., the required number of observations of the objective increases exponentially with the dimension. In many practical applications there exist first-principles models of the process to be optimized usually available in the form of simulations. There are multiple methodologies to make use of this additional information. Meta learning, see \cite{vanschoren2018}, can be applied to determine prior distributions as done by \cite{rothfuss21a,pan24b}. These priors can be seen as a warm start for the optimization process. Another approach is to apply \gls{bo} on the simulation and the real process jointly which is in literature known as \gls{mtbo}, initially proposed by \cite{Swersky2013} and successfully applied in practice, see \cite{Taylor2023}. The foundation of this approach, multi-task \gls{gp} prediction, was originally introduced by \cite{Bonilla2007}. This technique uses correlation matrices to capture the influence between various tasks, which are learned online from the available data.

 A final, critical consideration for practical application is the enforcement of safety guarantees. For example, when tuning controllers, the search space must be constrained to parameters that yield a stable closed-loop system. Especially if the response of the system is fast this safety mechanism can barely be handled externally by a fallback mechanism as for the \gls{lbsync} at \gls{euxfel}, see \cite{Schulz2015}. The primary challenge is that the system dynamics, and therefore the safe parameter region, are initially unknown. This necessitates estimating the safe region from data, for which uniform error bounds defined as
\begin{align}
    \label{eq:uniform_error_bound}
    \mathbb{P}\{|f(\x)-\mu(\x)|\leq \beta^\frac{1}{2}\sigma(\x),~\forall \x\in\mathcal{X}\}\geq 1-\delta,
\end{align} 
are employed.
The error bound states that the deviation of the unknown function $f$ can be bounded from the posterior mean
$\mu$ by scaling the posterior standard deviation $\sigma$ with the factor $\beta^\frac{1}{2}$. Depending on the assumption on $f$ different scaling factors can be obtained, see \cite{Snirivas2010,chowdhury17a} for frequentist perspective and \cite{Lederer2019} for a Bayesian perspective.

In the context of \gls{mtbo}, safety guarantees are more complex because the correlation matrices are uncertain. While other hyperparameters of the \gls{gp} may also be unknown, they can typically be selected conservatively. \cite{Luebsen2025} provide an overview of how to estimate this correlation uncertainty and incorporate it to derive a new scaling factor $\bar{\beta}$. Their guarantees were derived under the consideration of an \gls{icm}, which will be recalled in \cref{subsec:intrinsic_correlation_matrices}. A natural extension of \gls{icm} is the \gls{lmc} which is a composition of multiple \gls{icm}. This allows for a more flexible modeling of the inter-task correlations, as it can capture multiple features. 

In this manuscript, we derive safety guarantees for the \gls{lmc} under a Bayesian point of view. In addition, we show through a numerical comparison the improved empirical performance of using \glspl{lmc} compared to \glspl{icm} in \gls{mtbo}. 
\section{Preliminaries}
\label{sec:preliminaries}
In vanilla \gls{bo}, see \cref{fig:bo_loop}, \glspl{gp} are used to model an unknown scalar-valued objective function \(f:\mathcal{X} \to \mathbb{R}\), where the domain \(\mathcal{X} \subset \mathbb{R}^d\) is compact. It is assumed that the unknown function $f$ is continuous, which ensures an arbitrary good approximation of $f$ using universal kernels whose underlying \gls{rkhs} is dense in the set of continuous functions. In real applications the exact function values are typically not accessible, rather noisy observations are made. This behavior is modeled by additive Gaussian noise \(\epsilon \sim \mathcal{N}(0,\sigma_n^2)\), i.e., $y = f(\x)+\epsilon$, where $y$ is the measured value and \(\sigma_n^2\) denotes the noise variance.
Furthermore, we define the set of observations by \(\mathcal{D}\coloneqq\left\{(\x_k,y_k), k=1,\dots, N\right\}\) which is composed of the evaluated inputs combined with the corresponding observations. This set can be considered as the training set. A more compact notation of all inputs of \(\mathcal{D}\) is given by the matrix \(X = [\x_{1}, \dots, \x_{N}]^\top\) and of all observations by the vector \(\bm{y}=[y_1,\dots,y_N]^\top\). With this data set, the \gls{gp} creates a probabilistic surrogate model to predict $f(\x),~\x\in\mathcal{X}$. These predictions serve as inputs for an acquisition function $\alpha(\x)$, which identifies new promising inputs likely to minimize the objective.
Some common choices for acquisition functions are \gls{ucb}, (log) expected improvement \cite{Jones1998}, or predictive entropy search \cite{hernandez14}.
If safety constraints are considered, the optimization problem can be formulated as
\begin{align*}
    \min_{\x\in\mathcal{X}} & ~f(\x) \\
    \text{s.t.} & ~g(\x)\leq 0,
\end{align*} where $g(\x)$ is an unknown constraint function, usually modeled by a separate \gls{gp}. 

\subsection{Gaussian Processes and Reproducing Kernel Hilbert Spaces}
\label{subsec:bayesian_optimization}

A \gls{gp} is fully defined by a mean function \(m(\x)\) and a kernel \(k(\x,\x^{\prime}): \mathcal{X}\times\mathcal{X} \to \mathbb{R}\). In the context of \glspl{gp}, the kernel is also referred to as a covariance function. The \gls{gp} encodes prior knowledge of the unknown function $f$, i.e., $p(f(\x)) = \mathcal{GP}(m(\x),k(\x,\x^{\prime}))$ where the $m(\x)$ may be replaced by the zero function without loss of generality. The kernel determines the dependency between function values at different inputs which is expressed by the covariance operator 
\begin{align}
\mathrm{cov}(f(\x),f(\x^{\prime}))=k(\x,\x^{\prime}).
\label{eq:single_task_cov}
\end{align} Commonly used kernels are the spectral mixture \cite{wilson13}, Mat\'ern \cite{Genton01} or squared exponential kernel, where the latter is defined as $k_\mathrm{SE}(\x,\x^{\prime}) = \sigma_f^2\exp\left(-\frac{1}{2}(\x-\x^{\prime})^\top\Delta^{-2}(\x-\x^{\prime})\right)$ with \(\Delta = \mathrm{diag}(\bm{\vartheta})\).
The signal variance \(\sigma_f^2\), the lengthscales \(\bm{\vartheta}\) and the noise variance \(\sigma_n^2\) constitute the hyperparameters, allowing for adjustments of the kernel.

Given the set of observations and the prior of $f$, the posterior $p(f|X,\bm{y}) = \mathcal{GP}(\mu(\x),\sigma^2(\x))$ can be computed by applying Bayes' rule.
As shown by \cite{Williams2006} the posterior is also Gaussian given by
\begin{equation}
\label{eq:posterior_gp}
\begin{aligned}
        \mu(\x) &= K_X(\x)\left(K+\sigma_n^2I\right)^{-1}\bm{y}\\
        \sigma^2(\x) &= k(\x,\x)-K_X(\x)\left(K+\sigma_n^2I\right)^{-1}K_X(\x)^\top,   
\end{aligned}
\end{equation}     
where $K_X(\x) = K(\x,X)$ and $K = K(X,X)$ is the Gram matrix of the training data.
\begin{figure}[t]
    \centering

\begin{tikzpicture}[
  node distance=1.8cm and 2.5cm,
  box/.style={draw, thick, minimum width=3.2cm, minimum height=1.5cm, align=center},
  arrow/.style={->, thick, >=Latex},
  label/.style={font=\small, midway, fill=white, inner sep=1pt}
]

\node (top_left) {};
\node (blackbox) [draw, rounded corners=3pt,  right=1.5cm of top_left, fill=black, text=white, inner sep=8pt, outer sep=0pt] {Black Box};
\node [circle, draw, thick, inner sep=1pt, right=2cm of blackbox] (plus) {+};
\node (top_right) [right=4cm of blackbox] {};
\node (lower_left) [below=1cm of top_left] {};
\node (lower_right) [below=1cm of top_right] {};
\node (gp) [draw, rounded corners=3pt,inner sep=4pt, outer sep=0pt, left=.5cm of lower_right] {Gaussian Process};
\node (acq) [draw, rounded corners=3pt,inner sep=4pt, outer sep=0pt, left=2cm of gp] { Acq Function};

\draw[arrow] (top_left.center) -- (blackbox) node[midway,above, font=\scriptsize] {Input} node[midway,below, font=\scriptsize] {$x$};
\draw[arrow] (blackbox) -- (plus) node[midway,above, font=\scriptsize] {Output} node[midway,below, font=\scriptsize] {$f(x)$};
\draw (plus) -- (top_right.center) node[midway,above, font=\scriptsize] {Measurement} node[midway,below, font=\scriptsize] {$y=f(x) + \varepsilon$};
\draw[arrow] (top_right.center) |- (gp);
\draw[arrow] (gp) -- (acq) node[midway,above, font=\scriptsize] {Posterior} node[midway,below, font=\scriptsize] {$\mu(x), \sigma^2(x)$};
\draw (acq) -| (top_left.center);
\draw[arrow] ($(plus)+(0,1.)$) -- (plus) node[midway,below, font=\scriptsize, rotate=90]{$\varepsilon$} node[midway,above, font=\scriptsize, rotate=90] {Noise};

\end{tikzpicture}
    \caption{Schematic representation of the Bayesian optimization loop.}
    \label{fig:bo_loop}
\end{figure}
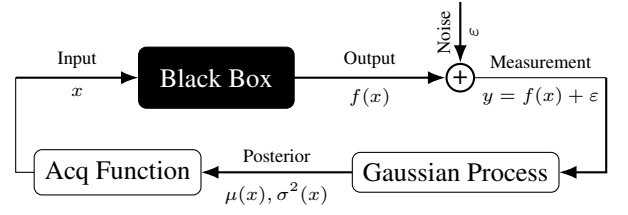

\subsection{Intrinsic Co-regionalization Model}
\label{subsec:intrinsic_correlation_matrices}
In this section, we quickly recapitulate the \gls{icm} which is mainly used in the context of multi-task \gls{bo}. To do so, we start with a short introduction to vector-valued \glspl{rkhs} and then continue with the safety guarantees for uncertain correlation matrices from frequentist and Bayesian perspectives.

The \gls{icm} describes the internal correlation structure between different outputs of vector-valued functions. Consider $f:\mathcal{X}\to\mathbb{R}^u$ which is a vector-valued function from some compact set $\mathcal{X}\subset \mathbb{R}^d$, that can be decomposed into a mixture of $m\geq u$ latent functions, i.e., 
\begin{align*}
    \f(\x) = B^\top \h(\x) = \sum_{i=1}^{m} B^{(i)} h^{(i)}(\x),
\end{align*}
where all $h^{(i)}\in\hil$ belong to the same \gls{rkhs}, $B\in\mathbb{R}^{m\times u}$ and $B^{(i)}$ denotes the i\ts{th} row of $B$. The underlying kernel is defined as 
$K(\x,\x^{\prime},t,t^{\prime}) = \Sigma_{t,t^{\prime}} k(\x,\x^{\prime}),$ 
where $\Sigma = B^\top B$ is the correlation matrix with $t,t^{\prime}\in\{1,\dots,u\}$ denote the tasks, and $k(\x,\x^{\prime})$ is the scalar-valued basis kernel. Kernels with separable task and input components belong to the class of separable kernels which are mostly used in literature \cite{Luebsen2024,Luebsen2025,Swersky2013,Taylor2023}.  

The \gls{rkhs} for vector-valued functions is induced by the tensor vector space $\mathbb{R}^u\bigotimes \hil$ spanned by elements of the form $\w\otimes \h$ with $\w\in\mathbb{R}^u$ and $h\in\hil$, see \cite{Caponnetto2008}. All elements $\w_1\otimes \h_1,\w_2\otimes \h_2\in\mathbb{R}^u\bigotimes \hil$ with constant $c\in\mathbb{R}$ satisfy the multilinear relation
 \begin{align*}
     c(\w_1+\w_2)\otimes \h_1 &= (\w_1+\w_2) \otimes c\h_1\\ &= c(\w_1\otimes \h_1)+c(\w_2\otimes \h_1).
 \end{align*}
 To obtain a dot product space, we define 
 \begin{align*}
     \langle \w_1\otimes \h_1,\w_2\otimes \h_2\rangle = \langle \w_1,\w_2\rangle\langle \h_1,\h_2\rangle_\hil,
 \end{align*}
 and take the completion with respect to the induced norm of the elements in $\mathbb{R}^u\bigotimes \hil$. Thus, the tensor vector space $\mathbb{R}^u\bigotimes \hil$ becomes a Hilbert space $\bhil$.

We denote $\bvphi \colon  \mathcal{X} \to \hil$ as the feature map of the base kernel $k(\x,\x^{\prime})$.
The feature map of $\bhil$ is defined as
\begin{equation}
\label{eq:feature_map_MT}
\begin{aligned}
\Phi \colon \mathcal{X} &\to \mathcal{L}(\mathbb{R}^u, \bhil) \\
\Phi(\x) \colon \w &\mapsto B \w \otimes \bvphi(\x)
\end{aligned}
\end{equation}
for all $\w \in \mathbb{R}^u$, and the adjoint feature map as
\begin{equation}
\label{eq:adjoint_feature_map_MT}
\begin{aligned}
\Phi^\star \colon \mathcal{X} &\to \mathcal{L}(\bhil, \mathbb{R}^u) \\
\Phi^\star(\x) \colon \w \otimes \h &\mapsto B^\top \w \langle\bvphi(\x), \h\rangle_\hil
\end{aligned}
\end{equation}
for all $(\w \otimes \h) \in \bhil$, where $\mathcal{L}(\mathbb{R}^u,\bhil)$ denotes the set of linear operators from $\mathbb{R}^u$ to $\bhil$, and $\mathcal{L}(\bhil,\mathbb{R}^u)$ is defined analogously. With these definitions, it is easy to see that
\begin{align*}
\langle\Phi_\Sigma(\x), \Phi_\Sigma(\x')\rangle_{\bhil_\Sigma} &= \Phi_\Sigma^\star(\x)\Phi_\Sigma(\x') \\
&= B^\top B \langle \bvphi(\x), \bvphi(\x')\rangle_\hil \\
&= \Sigma k(\x, \x').
\end{align*}
The subscript \(\Sigma\) indicates that the feature map corresponds to the \gls{rkhs} induced by the correlation matrix \(\Sigma\).

The evaluation functional can be represented by the dot product
 \begin{align*}
    \f(\x) &= \langle\Phi(\x),\f_w\otimes \f_h\rangle_{\bhil_\Sigma}\\ &= B^\top \f_w\langle\bvphi(\x),\f_h\rangle_\hil\\ &= B^\top \f_w f(\x), 
 \end{align*}
 which is denoted as the reproducing property. We use the notation $\langle\cdot,\cdot\rangle_{\bhil_\Sigma}$ to denote in which \gls{rkhs} the inner product is taken.
 Note that since vector-valued functions are considered, the posterior mean $\bmu(\x)$ and variance $\bsig(\x)$ are vector-valued which is emphasized by the bold notation. Moreover, we will use the notation $\bmu_\Sigma(\x)$ and $\bsig_\Sigma(\x)$ to indicate that both functions live in $\bhil_\Sigma$.

\subsection{Safety Guarantees}
The central challenge in ensuring safety for multi-task \gls{bo} is the uncertainty surrounding the correlation matrix $\Sigma$. Our approach addresses this by assuming the true matrix is contained within a known set, $\mathcal{C}_\rho$, with a probability of $1-\rho$. In a Bayesian context, this set can be inferred from data using approximation methods like \gls{mcmc} sampling, as demonstrated by \cite{Luebsen2024,Luebsen2025}. We will now review the multi-task safety guarantees for the \gls{icm} case from this Bayesian perspective. To this end, we impose the following assumption.

 \begin{assum}
    \label{assum:sample_of_GP}
        The unknown vector-valued function $\f:\mathcal{X}\to\mathbb{R}^u$ is a sample from a \gls{gp} with zero mean, multi-task kernel $K(\x,\x^{\prime}) = \Sigma k(\x,\x^{\prime})$ and hyper-prior $\Sigma\sim p(\Sigma)$ with compact support, i.e., $\f(\x)\sim\mathcal{GP}(\bm{0},K(\x,\x^{\prime}))$.  The base kernel $k(\x,\x^{\prime})$ is at least four times partially differentiable on $\mathcal{X}$.
\end{assum}

The differentiability assumption on the kernel is necessary to ensure the existence of the Lipschitz constant $L_f$ of the sample paths of the \gls{gp}. Moreover, the moduli of continuity $\omega_\mu$ defined as $ |\bmu(\x)-\bmu(\y)| \leq \omega_{\mu}(\|\x-\y\|_p)$, and $\omega_\sigma$ defined as $|\bsig(\x)-\bsig(\y)| \leq \omega_{\sigma}(\|\x-\y\|_p)$ of the posterior mean and standard deviation, respectively, are required. Both quantities can be upper bounded using solely the Lipschitz constant of the kernel, see \cite{Luebsen2025} for details.

Furthermore, from \cite{Luebsen2025}, we know that for any symmetric positive definite $\Sigma$ and $\Sigma^{\prime}$ the underlying \glspl{rkhs} are equivalent. This implies that every element $\f\in\bhil_\Sigma$ also belongs to $\bhil_{\Sigma^{\prime}}$ and vice versa. In particular, every element $\f\in\bhil_\Sigma$ can be represented in $\bhil_{\Sigma^{\prime}}$ in terms of a linear operator $\widetilde{L}$, such that 
\begin{align} 
    \bhil_\Sigma = \widetilde{L}\bhil_{\Sigma^{\prime}}:= \{\widetilde{L} \f : \f \in \bhil_{\Sigma^\prime}\}.
\label{eq:equivalent_rkhs}
\end{align}
This property holds because both \glspl{rkhs} share the same single task \gls{rkhs} $\hil$ differing only in their mixture matrix $B$. Provided that $B$ is injective, the operator $\widetilde{L}$ and its inverse are well-defined. The operator $\widetilde{L}$ allows to describe a function $\f\in\bhil_\Sigma$ as a function $\f^\prime\in\bhil_{\Sigma^\prime}$, a relationship that is central to the proof of the following theorem.
The following theorem is due to \cite{Luebsen2025}.

 \begin{thm}
    \label{th:robust_scaling_factor}
    \hfill
    Given Assumption~\ref{assum:sample_of_GP}, $\rho>0$ and $\delta > 0$. Select any $\Sigma^{\prime}\in\mathcal{C}_{\rho}$ which should be used for inference and let 
    \begin{alignat*}{1}
        &\beta=2\log\left(\frac{|\mathcal{I}|}{\delta}\right)\\ &\psi=L_f\tau+\omega_\mu(\tau)+\beta^\frac{1}{2}\omega_{\sigma}(\tau)\\
        &\nu^2 = \max_{\Sigma\in\mathcal{C}_\rho}\bigg\{\|\bmu_{\Sigma'} - \widetilde{L}^\star \bmu_\Sigma\|_{\bhil_{\Sigma'}}^2\\ & \hspace{1cm} + \frac{1}{\sigma_n^2} \sum_{n=1}^N \|\bmu_{\Sigma}(\x_n)-\bmu_{\Sigma^{\prime}}(\x_n)\|_2^2\bigg\}\\
        &\gamma^2 =  \max_{\Sigma\in\mathcal{C}_\rho}\|\Sigma'^{-1}\Sigma\|_2,
    \end{alignat*}
    where $\widetilde{L}^\star$ is the adjoint operator of $\widetilde{L}$ as defined in \eqref{eq:equivalent_rkhs}, and $|\mathcal{I}|$ is the cardinality of the discretization of $\mathcal{X}$ with mesh size $\tau$.
    Then, with $\bar{\beta} = (\nu+\gamma\beta^\frac{1}{2})^2$ we have with probability $(1-\delta)(1-\rho)$
    \begin{align*}
        |\f(\x)-\bmu_{\Sigma^{\prime}}(\x)|\leq \bar{\beta}^\frac{1}{2}\bsig_{\Sigma^{\prime}}(\x)+\psi, 
        \end{align*}
    for all $\x\in\mathcal{X}$.
\end{thm}
    Note that the operator $L$ used in the original theorem from \cite{Luebsen2025}, corresponds to $\widetilde{L}^\star$ in our definition in \eqref{eq:equivalent_rkhs}.
\begin{proof}
    The proof of this theorem is quite long and involves multiple steps in between. We provide an outline in \ifarxiv{\cref{sec:proof_outline}} \else{\cite{Luebsen2026}}\fi
\end{proof}

Theorem~\ref{th:robust_scaling_factor} shows that the uniform error bound in \eqref{eq:uniform_error_bound} can be preserved for uncertain correlation matrices by adjusting the scaling factor from $\beta$ to $\bar{\beta}$. As illustrated in the proof outline, the role of $\nu$ is to bound the deviation of the posterior mean when different correlation matrices are considered. 
An equivalent role is taken by $\gamma$ which bounds the deviation of the posterior standard deviation.

\section{Safe Bayesian Optimization for LMCs}

\subsection{The Linear Model of Co-Regionalization}
\label{subsec:lmc}
The expressivity of the \gls{icm} is limited, as it only allows for a single shared feature across tasks. Consequently, the model cannot capture complex task correlations that arise when different tasks are governed by distinct features. This misspecification can degrade posterior predictions and lead to overly conservative confidence sets. The \gls{lmc} from \cite{Alvarez2011} is a natural extension of the \gls{icm} as it allows for a more flexible modeling of the inter-task correlations. In general, the \gls{lmc} assumes that the function is composed as 
\begin{align}
    \label{eq:lmc}
    \f(\x) = \sum_{i=1}^{H} B_i^\top \h_i(\x) = \sum_{i=1}^{H} \sum_{j=1}^{m} B^{(j)}_i h^{(j)}_i(\x),
\end{align} 
where $\f(\x):\mathcal{X}\to \mathbb{R}^u$, $B_i\in\mathbb{R}^{m\times u}$ are injective mixture matrices and $h^{(j)}_i\in\hil_i$ belong to some \gls{rkhs}. In accordance to \eqref{eq:lmc}, the kernel function of the underlying \gls{rkhs} is defined as 
\begin{align}
\label{eq:lmc_kernel}
    &\K_\Sigma(\x,\x^{\prime}) = \sum_{i=1}^H \Sigma_i k_i(\x,\x^{\prime}),
\end{align}
where $\Sigma_i = B_i^\top B_i$ are the correlation matrices and $k_i(\x,\x^{\prime})$ are the basis kernels. The subscript $\Sigma$ indicates the dependence of the kernel on the correlation matrices $\Sigma_i$. Accordingly, the subscript $\Sigma^{\prime}$ is used if the correlation matrices $\Sigma_i^{\prime} = B_i^{\prime \top} B_i^{\prime}$ are used.
For notational convenience, we overload the notation $\Sigma$ in the \gls{lmc} case to denote 
\begin{align*}
    \Sigma = (\Sigma_1,\dots,\Sigma_H),
\end{align*}
and $\Sigma^{\prime}$ analogously.
The \gls{rkhs} corresponding to $\K_\Sigma$ is denoted by $\bhilh_\Sigma = \bigoplus_{i=1}^{H}\bhil_{\Sigma_i}$, where $\bigoplus$ denotes the direct sum of Hilbert spaces. In particular, every element $f\in\bhilh$ can be written as
 \[\f = \oplus_{i=1}^H \f_i = \oplus_{i=1}^H \f_{w,i} \otimes \f_{h,i}.\] All properties of $\bhil$ in \eqref{eq:feature_map_MT} and \eqref{eq:adjoint_feature_map_MT} apply for $\bhilh$. The inner product of $\bhilh$ is induced by the inner products of the spaces $\bhil_i$ such that
\begin{align*}
    \langle\cdot,\cdot\rangle_{\bhilh_\Sigma} = \sum_{i=1}^{H} \langle\cdot,\cdot\rangle_{\bhil_{\Sigma_i}} = \sum_{i=1}^{H} \langle\cdot,\cdot\rangle_{\mathbb{R}^u}\langle\cdot,\cdot\rangle_{\hil_i}.
\end{align*}

To perform inference \eqref{eq:posterior_gp} for the multi-task setting, we define the new multi-task Gram matrix as
\begin{align}
\label{eq:multi_task_gram_matrix}
\Gamma_\Sigma = \Gamma_\Sigma(\tilde{X},\tilde{X})= 
\begin{bmatrix} 
\Gamma_{1,1} & \dots & \Gamma_{1,u}\\
\vdots & \ddots & \vdots \\
\Gamma_{u,1} & \dots & \Gamma_{u,u}
\end{bmatrix},    
\end{align}
where $\tX= [X_1^\top,\dots,X_u^\top]$, $\ty = [\y_1^\top,\dots,\y_u^\top]$ and $\Gamma_{t,t^{\prime}}$ are the Gram matrices using data from tasks $t$ and $t^{\prime}$, in particular 
\[ \Gamma_{t,t^{\prime}} = \sum_{i=1}^{H} \sigma^2_{i_{t,t^\prime}} k_i(X_t,X_{t^{\prime}})\] with $\sigma^2_{i_{t,t^\prime}}$ being the $(t,t^{\prime})$ entry of $\Sigma_i$. 

\subsection{Safe Bayesian Optimization for LMCs}
From Theorem~\ref{th:robust_scaling_factor} we can see that the regularity characteristics given by $\omega_\mu$ and $\omega_\sigma$ change since they depend on the Lipschitz constant of the kernel $K_\Sigma$. For the \gls{icm} case we can see from \cite{Luebsen2025} that this is given by $L_K = q L_k$ where $L_k$ is the Lipschitz constant of the base kernel $k$ and $q = \max_i \Sigma_{ii}$; the maximum of the diagonal entry of $\Sigma$. Hence, one can see that the upper bounds on $\omega_\mu$ and $\omega_\sigma$ can be easily extended by taking in addition the maximum of all feature correlation matrices $\Sigma_j$. Since $L_f$ can be extended similarly or computed numerically, we exclude these details as they comprise purely technical algebraic manipulations.
More interesting changes appear in the terms $\nu$ and $\gamma$, in particular, in the definition of the linear operator $\widetilde{L}$. In the following lemma, we extend the definition of $\widetilde{L}$ for the \gls{lmc} denoted by $L$
\begin{lem}
    \label{lemma:linear_oper}
        There is a linear operator $L:\bhilh^{\prime} \to \bhilh$ such that
        \begin{enumerate}[(i)]
            \item $\bhilh = L\bhilh^{\prime}$,
            \item $L\f = \oplus_{i=1}^H B_i {B_i^\prime}^{-1}\f_{w,i}\otimes \f_{h,i}$,
        \end{enumerate}
\end{lem}
\begin{proof}
    \begin{enumerate}[(i)]
        \item\label{enum:L_1} We note that both $\f$ and $\f^{\prime}$ lie in the span composed of the base \glspl{rkhs} $\hil_i,~i=1,\dots,H$ and the column space of the matrices $B_i$ and $B_i^\prime$, respectively. Since, the matrices are injective, there is a bijection between the column spaces of $B_i$ and $B_i^\prime$.
        \item  We define $\bhilh_0$ as the \gls{lmc} \gls{rkhs} with $B_i = I$ for all $i$. It is clear that there is an operator $L_\Sigma$ such that $\bhilh_\Sigma = L_\Sigma \bhilh_0$. In particular, $L_\Sigma \f = \oplus_{i=1}^H B_i\f_{w,i}\otimes \f_{h,i}$. Analogously, there is an operator $L_{\Sigma^{\prime}}$ such that $\bhilh_{\Sigma^{\prime}} = L_{\Sigma^{\prime}}\bhilh_0$. Due to injectivity of $B_i$ and $B_i^\prime$, the operators $L_\Sigma$ and $L_{\Sigma^{\prime}}$ are invertible. Hence, we define 
        \begin{align*}
            L\f &= L_{\Sigma} L_{\Sigma^{\prime}}^{-1}\f\\
            &= \oplus_{i=1}^H B_i  {B_i^\prime}^{-1}\f_{w,i}\otimes \f_{h,i},\quad \forall\, \f\in \bhilh_{\Sigma^\prime}.
        \end{align*}
    \end{enumerate}
\end{proof}

It remains to show that the term $\nu$ in Theorem~\ref{th:robust_scaling_factor} can be computed for the \gls{lmc} by simply replacing $\widetilde{L}$ by $L$ from Lemma~\ref{lemma:linear_oper}.
\begin{lem}
    \label{lemma:nu_lmc}
    For the \gls{lmc} case, it holds that
    \begin{align*}
        |\bmu_{\Sigma}(\x) - \bmu_{\Sigma'}(\x)| \leq \nu \bsig_{\Sigma^{\prime}}(\x),\quad \forall \x\in\mathcal{X},~\Sigma\in\mathcal{C}_\rho,
    \end{align*}
    where  
    \begin{align*}
        \nu^2 &= \max_{\Sigma\in\mathcal{C}_\rho}\bigg\{\|\bmu_{\Sigma^{\prime}} - L^\star \bmu_\Sigma\|_{\bhilh_{\Sigma^{\prime}}}^2\\ & \hspace{1cm} + \frac{1}{\sigma_n^2} \sum_{n=1}^N \|\bmu_{\Sigma}(\x_n)-\bmu_{\Sigma^{\prime}}(\x_n)\|_2^2\bigg\}.
    \end{align*}
\end{lem}
\begin{proof}
    The proof is given in \ifarxiv{\cref{sec:proof_nu_lmc}} \else{\cite{Luebsen2026}}\fi
\end{proof}

 The final ingredient to extend Theorem~\ref{th:robust_scaling_factor} is to compute $\gamma$. This is summarized in the following lemma.

\begin{lem}
    \label{lemma:gamma_lmc}
    Let $\gamma^2 = \max_{\Sigma\in\mathcal{C}_\rho}\max_{i=1,\dots,H}\|\Sigma_i {\Sigma_i^\prime}^{-1}\|_2$, then $\bsig_{\Sigma}^2(\x)\leq \gamma^2 \bsig_{\Sigma^{\prime}}^2(\x)$ for all $\x\in\mathcal{X}$ and $\Sigma\in\mathcal{C}_\rho$. 
\end{lem}
\begin{proof}
    From \cite{Capone2022} lemma A.4 we can conclude that if $\K_\Sigma \leq \K_{\Sigma^\prime}$ then $\bsig_{\Sigma}^2(\x)\leq \bsig_{\Sigma^\prime}^2(\x)$ for all $\x\in\mathcal{X}$.
    As shown in \eqref{eq:lmc_kernel}, the kernels can be written as $\K_\Sigma(\x,\x^{\prime}) = \sum_{i=1}^H \Sigma_i k_i(\x,\x^{\prime})$ and $\K_{\Sigma^\prime}(\x,\x^{\prime}) = \sum_{i=1}^H \Sigma_i^{\prime} k_i(\x,\x^{\prime})$. Furthermore, 
    let $\gamma_i\geq \|\Sigma_i {\Sigma_i^\prime}^{-1}\|$ then we have $ \Sigma_i \leq \gamma_i^2 \Sigma_i^{\prime}$ which implies $\Sigma_i k_i(\x,\x^{\prime}) \leq \gamma_i^2 \Sigma_i^{\prime} k_i(\x,\x^{\prime})$. Hence, $\max_{i=1,\dots,H}\gamma_i^2$ holds for all features. Finally, we take the maximum over all $(\Sigma_i)_{i=1}^H\in\mathcal{C}_\rho$ to conclude the proof.
\end{proof}

Alternatively, each feature can be scaled by an individual $\gamma_i^2$ which leads to a less conservative bound.

\begin{thm}
    Theorem~\ref{th:robust_scaling_factor} holds for the \gls{lmc} case with
    \begin{align*}
        \nu^2 &= \max_{\Sigma\in\mathcal{C}_\rho}\|\bmu_{\Sigma'} - L^\star \bmu_\Sigma\|_{\bhilh_{\Sigma'}}^2\\ & \hspace{1cm} + \frac{1}{\sigma_n^2} \sum_{n=1}^N \|\bmu_{\Sigma}(\x_n)-\bmu_{\Sigma^{\prime}}(\x_n)\|_2^2,\\
        \gamma^2 &= \max_{\Sigma\in\mathcal{C}_{\rho}}\max_{i=1,\dots,H}\|\Sigma_i {\Sigma_i^\prime}^{-1}\|_2.
    \end{align*}
\end{thm}

\begin{proof}
    With Lemma~\ref{lemma:linear_oper}, we can compute $\nu$ in Theorem~\ref{th:robust_scaling_factor} for the \gls{lmc} case by substituting $\widetilde{L}$ by $L$ and take the norms with respect to $\bhilh$. 
    Finally, Lemma~\ref{lemma:gamma_lmc} provides the definition of $\gamma$ for the \gls{lmc} case.
\end{proof}

\section{Simulation}
This section compares the \gls{icm} and \gls{lmc} on a synthetic benchmark ($u=2$) over a four-dimensional input space and squared exponential base kernels, see \cite{Luebsen2025b} for the code. We sample objective functions from a \gls{gp} with a two-feature \gls{lmc} kernel with a dominant, highly correlated feature and a minor, uncorrelated one (see \cref{tab:sim_params} for specific kernel parameters). Due to the fine discretization grid ($\tau=0.001$), we neglect $\psi$ in the computation of the scaling factor $\bar{\beta}$. To favor correlations, the hyperprior for the first feature is an \gls{lkj} distribution with parameter $(\eta = 0.1)$, while the second is fixed to the identity. The norm of the first \gls{rkhs} function is set to $(\|\cdot\|_{\bhil_1}=30.5)$ and the second to $(\|\cdot\|_{\bhil_2}=10.2)$. In the simulation the supplementary task is evaluated eight times for every main task evaluation. Full experimental parameters and safety thresholds are detailed in \cref{tab:sim_params}.

\begin{table}[h]
    \centering
    \caption{Simulation parameters for the synthetic benchmark comparison between \gls{icm} and \gls{lmc}.}
    \label{tab:sim_params}
    \begin{tabular}{@{}llc@{}}
        \toprule
        \textbf{Parameter} & \textbf{Symbol} & \textbf{Value} \\
        \midrule
        Input Space & $\mathcal{X}$ & $[-1,1]^4$ \\
        Number of Tasks & $u$ & $2$ \\
        \midrule
        Base kernels & $k_i(\cdot,\cdot)$ & Squared Exponential \\
        Dominant Feature Norm & $\|\cdot\|_{\bhil_1}$ & $30.5$ \\
        Dominant Feature Correlation & $r_1$ & $0.85$ \\
        Minor Feature Norm & $\|\cdot\|_{\bhil_2}$ & $10.2/6.5$ \\
        Minor Feature Correlation & $r_2$ & $0.0$ \\
        \midrule
        Discretization Grid & $\tau$ & $0.001$ \\
        Set Probability & $\rho$ & $0.05$ \\
        Failure Probability & $\delta$ & $0.05$ \\
        Safety Threshold & $T$ & $1$ \\
        \midrule
        Feature 1 Prior & -- & \gls{lkj} ($\eta = 0.1$) \\
        Feature 2 Prior & -- & Identity (Fixed) \\
        Repetitions & -- & $20$ \\
        Acq. function & -- & UCB \\
        \bottomrule
    \end{tabular}
\end{table}

\begin{figure}
    \centering
    \includegraphics[width=0.48\textwidth]{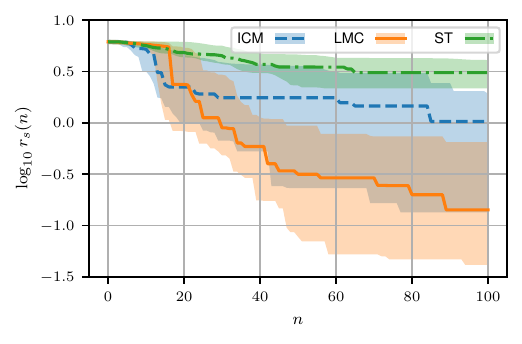}
    \caption{Convergence plot comparing \gls{lmc}, \gls{icm}, and single task safe \gls{bo} on the synthetic benchmark problem. The shaded areas represent the $25\%-75\%$ quantile and the solid lines the median over $20$ runs.}
    \label{fig:rkhs_samps_comp}
\end{figure}

\begin{figure}
    \centering
    \includegraphics[width=0.48\textwidth]{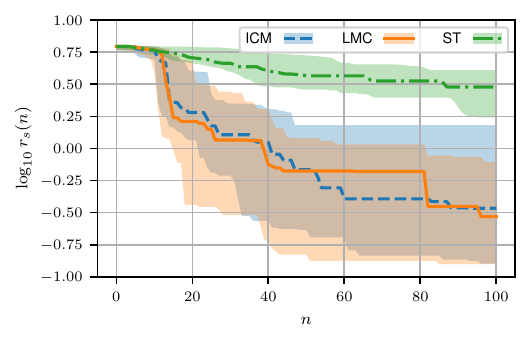}
    \caption{Convergence plot comparing \gls{lmc}, \gls{icm}, with the same setting as in \cref{fig:rkhs_samps_comp} but with a smaller norm of the second feature $(\|\cdot\|_{\bhil_2} = 6.5)$.}
    \label{fig:rkhs_samps_comp2}
\end{figure}

\cref{fig:rkhs_samps_comp} shows the convergence plot of the $\log_{10}$ simple regret $r_s(n) = f(\x^*) - f(\x_n^+)$ over the number of evaluations $n$ of the main task, where $\x^*$ is the global optimum and $\x_n^+$ the best safe point found after $n$ evaluations.
The results demonstrate that both \gls{mtbo} algorithms outperform the single-task safe \gls{bo} significantly. This is because the multi-task kernel identifies the dominant correlation feature and utilizes this to extend the safe region. This is possible due to the extra points drawn from the simulation which are not constrained to a safe region. In particular, the \gls{lmc} is able to separate the function into its two base features, whereas, the \gls{icm}, in contrast, must compromise by using a single correlation matrix for both features. This results in a weaker task correlation which increases $\bar{\beta}$ and makes the safe region more conservative. As shown in \cref{fig:rkhs_samps_comp2}, the performance of both multi-task models becomes nearly identical when the norm of the minor feature is reduced to $(\|\cdot\|_{\bhil_2} = 6.5)$.  The single-task safe \gls{bo} performs worst, as it becomes trapped in a local safe set. Conversely, both multi-task approaches successfully transfer knowledge from the simulation to the real task, which enables identification of a disconnected safe set when the task correlation is sufficiently strong. All safe \gls{bo} methods maintain safety throughout the optimization process with zero failures.

\section{Conclusion}
We extended safety guarantees for multi-task Bayesian optimization with uncertain correlation matrices from \glspl{icm} to \glspl{lmc}. This was achieved by extending the uniform error bounds to \glspl{lmc} and deriving the necessary components to compute a robust scaling factor. Finally, a numerical comparison on a synthetic benchmark demonstrated that \glspl{lmc} can provide superior performance over \glspl{icm}. Both multi-task approaches can outperform single-task safe \gls{bo} by transferring knowledge from the simulation to the real task. This also allows for an identification of a non-connected safe set. 
The main limitations of the \gls{lmc} as well as the \gls{icm} approach is the increased computational complexity for estimating the set $\mathcal{C}_\rho$ and computing the scaling factor $\bar{\beta}$.
Future work could focus on extending the theoretical analysis to other multi-task kernels beyond \glspl{lmc}, e.g., convolution processes and other non-separable kernels and merging approximation techniques for \glspl{gp} to reduce computational complexity.



\section*{DECLARATION OF GENERATIVE AI AND AI-ASSISTED TECHNOLOGIES IN THE WRITING PROCESS}
During the preparation of this work the author(s) used Gemini 2.5 Pro in order to check the language and get inspiration for the phrasing. After using this tool/service, the author(s) reviewed and edited the content as needed and take(s) full responsibility for the content of the publication.

\bibliography{biblio}
\ifarxiv
\appendix
\section{Supplementary Material}
\label{sec:supplementary}
The Gaussian process posterior mean and variance is given by
 \begin{align*}
     \bmu_\Sigma(\x) &= K(\x,\tX)(\K_\Sigma+\sigma_n^2I)^{-1}\ty,\\
     {\bsig^2_{\Sigma}}(\x) &= K(\x,\x) - K(\x,\tX)(\K_\Sigma+\sigma_n^2I)^{-1}K(tX,\x).
 \end{align*}
According to Mercers theorem single-task kernels can be expanded, \[k(\x,\x') = \sum_{i=1}^{N_\hil} \lambda_i \psi_i(\x)\psi_i(\x') = \bvphi^\star(\x) \bvphi(\x'),\] where $\lambda_i$ denote the eigenvalues. The same applies for multi-task kernels
\begin{align*}
    K(\x,\x') = \sum_{i=1}^{N_\hil} \lambda_i \Psi_i^\star(\x) \Psi_i(\x') = \Phi^\star(\x) \Phi(\x').
\end{align*}
Using this decomposition, the kernel matrices can be expressed as
\begin{align*}
    K_\Sigma &= \Phi_\Sigma^\star\Phi_\Sigma,\\
    K_\Sigma(\x,\tX) &= \Phi_\Sigma^\star(\x)\Phi_\Sigma,
\end{align*}
where $\Phi_\Sigma = [\Phi_\Sigma(\x_1),\dots,\Phi_\Sigma(\x_n)]$ denotes the stacked version of the element-wise evaluation of $\Phi_\Sigma$ at each training point.
Hence, the posterior mean and standard deviation can be rewritten as
\begin{align*}
    \bmu_\Sigma(\x) &= \Phi_\Sigma^\star(\x)\Phi_\Sigma(\Phi_\Sigma^\star\Phi_\Sigma+\sigma_n^2I)^{-1}\ty,\\
    \bsig^2_\Sigma(\x) &= \Phi_\Sigma^\star(\x)\Phi_\Sigma(\x) - \Phi_\Sigma^\star(\x)\Phi_\Sigma\\
    &\quad(\Phi_\Sigma^\star\Phi_\Sigma+\sigma_n^2I)^{-1}\Phi_\Sigma^\star\Phi_\Sigma(\x).
\end{align*}

The goal is now to represent the posterior mean in the space of the posterior kernel $K_\Sigma^P$, i.e., \[\bmu_\Sigma(\x) = \langle K_\Sigma^P(\x,\cdot), \bmu^P_{\Sigma}\rangle_{\bhil_\Sigma}.\] 
To do so, we use the following fact as shown in \cite{chowdhury17a}:
\begin{equation}
    \label{eq:trick1}
    \begin{aligned}
        &\Phi_\Sigma\left(\Phi_\Sigma^\star\Phi_\Sigma+\sigma_n^2I\right)^{-1}\Phi_\Sigma^\star\\  &\qquad = \left(\Phi_\Sigma\Phi_\Sigma^\star+\sigma_n^2I\right)^{-1} \Phi_\Sigma\Phi_\Sigma^\star.
    \end{aligned}
\end{equation}

First, we derive the evaluation functional $K_\Sigma^P(\x,\cdot)$ from the posterior variance,
\begin{align*}
    K_\Sigma^P(\x,\x') &= K_\Sigma(\x,\x') - K_\Sigma(\x(\K_\Sigma+\sigma_n^2I)^{-1}K_\Sigma\x')\\
    &= \Phi_\Sigma^\star(\x)\Phi_\Sigma(\x') - \Phi_\Sigma^\star(\x)\Phi_\Sigma\\
    &\quad(\Phi_\Sigma^\star\Phi_\Sigma+\sigma_n^2I)^{-1}\Phi_\Sigma^\star\Phi_\Sigma(\x')\\
    &= \Phi_\Sigma^\star(\x)(I - \Phi_\Sigma(\Phi_\Sigma^\star\Phi_\Sigma+\sigma_n^2I)^{-1}\\
    &\quad \Phi_\Sigma^\star)\Phi_\Sigma(\x')\\
    &= \Phi_\Sigma^\star(\x)(I - (\Phi_\Sigma\Phi_\Sigma^\star+\sigma_n^2I)^{-1}\\
    &\quad \Phi_\Sigma\Phi_\Sigma^\star)\\
\end{align*}
Expanding the identity operator by \hfill $ (\Phi_\Sigma\Phi_\Sigma^\star+\sigma_n^2I) (\Phi_\Sigma\Phi_\Sigma^\star+\sigma_n^2I)^{-1}$ leads to
\begin{align*}
    K_\Sigma^P(\x,\x') &= \sigma_n^2 \Phi_\Sigma^\star(\x)(\Phi_\Sigma\Phi_\Sigma^\star+\sigma_n^2I)^{-1}\Phi_\Sigma(\x').
\end{align*}
Hence, the evaluation functional is given by
\begin{align*}
    K_\Sigma^P(\x,\cdot) = \sigma_n(\Phi_\Sigma\Phi_\Sigma^\star+\sigma_n^2I)^{-\frac{1}{2}}\Phi_\Sigma(\x) 
\end{align*}
To express the posterior mean in terms of this evaluation functional, we first apply \eqref{eq:trick1} to obtain
\begin{align*}
    \bmu_\Sigma(\x) &= \Phi_\Sigma^\star(\x)\Phi_\Sigma(\Phi_\Sigma^\star\Phi_\Sigma+\sigma_n^2I)^{-1}\ty\\
     &=\Phi_\Sigma^\star(\x)(\Phi_\Sigma\Phi_\Sigma^\star+\sigma_n^2I)^{-1}\Phi_\Sigma\ty
\end{align*}
Thus, $\bmu_\Sigma(\x) = \langle K_\Sigma^P(\x,\cdot), \bmu_\Sigma^P\rangle_{\bhil_\Sigma}$ implies that \[\bmu^P_{\Sigma} =  (\Phi_\Sigma\Phi_\Sigma^\star+\sigma_n^2I)^{-\frac{1}{2}}\Phi_\Sigma\ty.\]
In contrast the prior \gls{rkhs} representation is given by \[\bmu_\Sigma =  (\Phi_\Sigma\Phi_\Sigma^\star+\sigma_n^2I)^{-1}\Phi_\Sigma\ty.\]

\section{Proof outline of Theorem~\ref{th:robust_scaling_factor}}
\label{sec:proof_outline}
 For a finite discrete set $\mathcal{I} \subset \mathcal{X}$ \cite{Snirivas2010} showed that for the given $\beta$ it holds that
     \begin{align} 
        \label{eq:multi_bayes_concentration}
        \mathbb{P}\left\{|\f([\x])-\bmu_\Sigma([\x])|\leq \beta^\frac{1}{2}\bsig_\Sigma([\x]), ~\forall[\x]\in\mathcal{I}\right\}\geq 1-\delta.
    \end{align}
    Since $\mathcal{X}$ is compact, for every $\tau>0$ the cardinality $|\mathcal{I}|$ of the discretization with mesh size $\tau$ is finite. The extension from the discrete set $\mathcal{I}$ to the entire compact set $\mathcal{X}$ is done by bounding the discretization error using the Lipschitz constant $L_f$ of the sample paths as well as the moduli of continuity $\omega_\mu$ and $\omega_\sigma$ of the posterior mean and standard deviation, respectively. This yields the additional term $\psi$. The procedure up to here is similar to the single task case by \cite{Lederer2019}.  
    
    For the multi-task case, the bound in \eqref{eq:multi_bayes_concentration} need to be adapted. More precisely, the scaling factor $\beta$ is replaced by a $\bar{\beta}$ such that the bound holds for all correlation matrices $\Sigma$ in $\mathcal{C}_\rho$. First one selects a nominal correlation matrix $\Sigma^{\prime}\in\mathcal{C}_\rho$ which is used for inference. For the posterior mean we can write $|\f(\x)-\bmu_\Sigma(\x)|\leq |\f(\x)-\bmu_{\Sigma'}(\x)|+|\bmu_{\Sigma'}(\x)-\bmu_\Sigma(\x)|$. 
    The latter term can be bounded by $\nu\bsig_{\Sigma'}(\x)$. 
    Next, the posterior standard deviation is bounded $\bsig_{\Sigma}(\x)\leq \gamma \bsig_{\Sigma'}(\x)$.

\section{Proof of Lemma~\ref{lemma:nu_lmc}}
\label{sec:proof_nu_lmc}
The proof follows the same steps as in Lemma 13 by \cite{Luebsen2025}. We denote $\bhilh^P_\Sigma$ as the corresponding \gls{rkhs} of the posterior kernel $K^P_\Sigma$. Additionally, we define the operator \begin{align*}
        T \colon& \bhilh^P_{\Sigma^{\prime}} \to \bhilh^P_{\Sigma},\\
        & \g \mapsto \Psi_{\Sigma}^{-\frac{1}{2}}L \Psi_{\Sigma^{\prime}}^{\frac{1}{2}}\g,
    \end{align*}
    where $\Psi_{\Sigma} = (\Phi_{\Sigma}\Phi^\star_{\Sigma} + \sigma_n^2 I)$ and
    $\Psi_{\Sigma^{\prime}} = (\Phi_{\Sigma^{\prime}}\Phi^\star_{\Sigma^{\prime}} + \sigma_n^2 I)$. The operator $\Phi_{\Sigma}= [\Phi_\Sigma(\x_i)]_{i=1}^N$ denotes the stacked feature map evaluated at the data points $\x_i\in\tX$. Note that we have $\Gamma_\Sigma = \Phi^\star_{\Sigma} \Phi_{\Sigma}$ where the position of the $\star$ is swapped. 
    Let \[\bmu_\Sigma(\x) = \langle K^P_\Sigma(\cdot,\x), \bmu^P_\Sigma\rangle_{\bhilh^P_\Sigma}.\] It can be shown, see Appendix~\ref{sec:supplementary}, that 
    \begin{align*} 
        K^P_\Sigma(\x,\cdot) &= \Psi^{\frac{1}{2}}_\Sigma \Phi_\Sigma(\x)\\
        & = \left(\Phi_\Sigma\Phi^\star_{\Sigma} + \sigma_n^2 I\right)^{-\frac{1}{2}} \Phi_\Sigma(\x).
    \end{align*}
    From the definition of $T$ we can see that $K_\Sigma^P(\x,\cdot) = T K_{\Sigma^{\prime}}^P(\x,\cdot)$.
    Hence, we have
    \begin{align*}
        &|\bmu_{\Sigma}(\x) - \bmu_{\Sigma^{\prime}}(\x)| \\
        &= |\langle K^P_\Sigma(\cdot),\bmu^P_\Sigma\rangle_{\bhilh^P_\Sigma} - \langle K^P_{\Sigma^{\prime}}(\cdot),\bmu^P_{\Sigma^{\prime}}\rangle_{\bhilh^P_{\Sigma^{\prime}}}|\\
        &= |\langle K^P_{\Sigma^{\prime}}(\cdot), \bmu^P_\Sigma - T^\star \bmu^P_{\Sigma^{\prime}}\rangle_{\bhilh^P_{\Sigma^{\prime}}}|\\
        &\leq \|K^P_{\Sigma^{\prime}}(\cdot)\|_{\bhilh^P_{\Sigma^{\prime}}} \|\bmu^P_\Sigma - T^\star \bmu^P_{\Sigma^{\prime}}\|_{\bhilh^P_{\Sigma^{\prime}}}\\
        &= \bsig_{\Sigma^{\prime}}(\x) \|\bmu^P_\Sigma - T^\star \bmu^P_{\Sigma^{\prime}}\|_{\bhilh^P_{\Sigma^{\prime}}}.
    \end{align*}
Using the linearity of inner products, we have
        \begin{align*}
            \|\bmu_{\Sigma'}^{P}-T^\star\bmu_{\Sigma}^P\|_{\bhil_{\Sigma'}^P}^2 &= 
            \underbrace{\|\bmu_{\Sigma'}^{P}\|^2_{\bhil_{\Sigma'}^P}}_{(1)} + 
            \underbrace{\|T^\star\bmu_{\Sigma}^P\|_{\bhil_{\Sigma'}^P}^2}_{(2)}\\ &- 
            \underbrace{2\langle \bmu_{\Sigma'}^{P},T^\star\bmu_{\Sigma}^P\rangle_{\bhil_{\Sigma'}^P}}_{(3)}.
        \end{align*}
        For (1), we use from \cite{Snirivas2010} that
        \begin{align*}
            \|\bmu_{\Sigma'}^{P}\|^2_{\bhil_{\Sigma'}^P} &\leq \|\bmu_{\Sigma'}\|_{\bhil_{\Sigma'}}^2 +\frac{1}{\sigma_n^2}\sum_{i=1}^N \|\bmu_{\Sigma'}(\x_i)\|_2^2.
        \end{align*}
        For (2), the term can be expanded as 
        \begin{align*}
            \|T^\star\bmu_{\Sigma}^P\|_{\bhil_{\Sigma'}^P}^2 &= \langle T^\star\bmu_{\Sigma}^P,T^\star\bmu_{\Sigma}^P\rangle_{\bhil_{\Sigma'}^P}\\
            &= \frac{1}{\sigma_n^2} \langle \Psi_{\Sigma'}^{\frac{1}{2}} L^\star \bmu_{\Sigma},\Psi_{\Sigma'}^{\frac{1}{2}} L^\star \bmu_{\Sigma}\rangle_{\bhil_{\Sigma'}^P}\\
            & = \frac{1}{\sigma_n^2} \langle L^\star \bmu_{\Sigma}, (\Phi_{\Sigma'}\Phi_{\Sigma'}^\top+\sigma_n^2I) L^\star \bmu_{\Sigma}\rangle_{\bhil_{\Sigma'}^P}\\
            &= \|L^\star \bmu_{\Sigma}\|_{\bhil_{\Sigma'}}^2+ \frac{1}{\sigma_n^2} \langle \bmu_{\Sigma},L\Phi_{\Sigma'}\Phi_{\Sigma'}^\top L^\star \bmu_{\Sigma}\rangle_{\bhil_{\Sigma'}}\\ 
            &= \|L^\star \bmu_{\Sigma}\|_{\bhil_{\Sigma'}}^2+ \frac{1}{\sigma_n^2}\sum_{i=1}^n \|\bmu_\Sigma(\x_i)\|_2^2.
        \end{align*}
        The last term, (3), can be expanded as 
        \begin{align*}
            &\langle \bmu_{\Sigma'}^{P},T^\star\bmu_{\Sigma}^P\rangle_{\bhil_{\Sigma'}^P} = \frac{1}{\sigma_n^2}\langle \bmu_{\Sigma'}^{P},\Psi_{\Sigma'}^{\frac{1}{2}} L^\star  \bmu_{\Sigma}\rangle_{\bhil_{\Sigma'}^P}\\
            \quad &= \frac{1}{\sigma_n^2}\langle \Psi_{\Sigma'}^{-\frac{1}{2}}\Phi_{\Sigma'}\tilde{\ty},\Psi_{\Sigma'}^{\frac{1}{2}} L^\star \bmu_{\Sigma}\rangle_{\bhil_{\Sigma'}}\\
            \quad &= \frac{1}{\sigma_n^2}\langle \Phi_{\Sigma'}\tilde{\ty},L^\star \bmu_{\Sigma}\rangle_{\bhil_{\Sigma'}}\\
            \quad &= \frac{1}{\sigma_n^2}\langle \Psi_{\Sigma'}^{-1}\Phi_{\Sigma'}\tilde{\ty},\Psi_{\Sigma'}L^\star \bmu_{\Sigma}\rangle_{\bhil_{\Sigma'}}\\
            \quad &= \frac{1}{\sigma_n^2}\langle \bmu_{\Sigma'},(\Phi_{\Sigma'}\Phi_{\Sigma'}^\top+\sigma_n^2I) L^\star \bmu_{\Sigma}\rangle_{\bhil_{\Sigma'}}\\
            \quad &= \langle \bmu_{\Sigma'},L^\star  \bmu_{\Sigma}\rangle_{\bhil_{\Sigma'}} + \frac{1}{\sigma_n^2}\sum_{i=1}^n |\bmu_\Sigma(\x_i)^\top\bmu_{\Sigma'}(\x_i)|^2. 
        \end{align*}
        Finally, we observe that the sum of all terms admit the quadratic equation given in the claim.

\section{Expressiveness of LMCs}
In this subsection we illustrate the differences between functions sampled from an \gls{icm} and \gls{lmc}. In the first figure, an \gls{icm} is used with high correlation, in the second an \gls{lmc} which is composed of the same strongly correlated feature as in the first figure and an uncorrelated, smaller feature. The dominant feature may capture the global behaviour shared by reality and simulation, as described by the first-principles model. By contrast, the smaller, uncorrelated feature may represent secondary effects, such as disturbances or friction, that cause local deviations and can shift the global optimum.
\begin{figure}[h]
    \centering
    \includegraphics{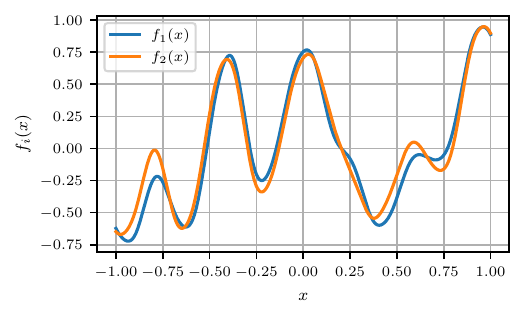}
    \caption{Functions generated from an \gls{icm} with high correlation.}
    \label{fig:mtrkhs}
\end{figure}
\begin{figure}[h]
    \centering
    \includegraphics{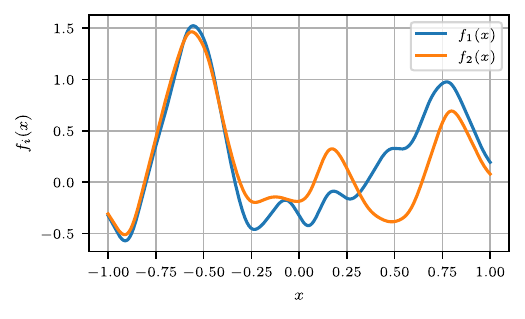}
    \caption{Functions generated from a \gls{lmc} with two base kernels. The first feature is the same as in \cref{fig:mtrkhs} and the second feature has a smaller \gls{rkhs} norm and no correlation between the tasks.}
    \label{fig:mixed_mtrkhs}
\end{figure}
\else
\fi
\end{document}